\documentclass[conference]{IEEEtran}
\IEEEoverridecommandlockouts

\usepackage{cite}
\usepackage{amsmath,amssymb,amsfonts,amsthm}
\usepackage{mathrsfs}
\usepackage{mdwmath}
\usepackage{mdwtab}
\usepackage{eqparbox}
\usepackage{cases}

\usepackage{graphicx}
\graphicspath{{../figures/}{./figures/}{./}}
\usepackage{epstopdf}
\usepackage{textcomp}
\usepackage{booktabs}
\usepackage{xcolor}
\usepackage{soul}
\usepackage{pifont}
\usepackage{subcaption}
\usepackage{balance}

\usepackage{algorithm,algorithmic}

\usepackage[english,vietnamese]{babel}

\usepackage{xspace}
\usepackage{hyperref}
\hypersetup{colorlinks=true,linkcolor=blue,citecolor=blue,urlcolor=blue}
\def\BibTeX{{\rm B\kern-.05em{\sc i\kern-.025em b}\kern-.08em
    T\kern-.1667em\lower.7ex\hbox{E}\kern-.125emX}}

\newtheorem{theorem}{Theorem}

\newcommand{\etal}{\textit{et al.}\xspace}
\makeatletter
\newcommand{\ScaleIfNeeded}{%
  \ifdim\Gin@nat@width>\linewidth\linewidth\else\Gin@nat@width\fi
}
\makeatother

\newcommand{\HeadCP}{\textsc{HeAD-CP}\xspace}
\newcommand{\DAPS}{\textsc{DAPS}\xspace}
\newcommand{\APS}{\textsc{APS}\xspace}
\newcommand{\SNAPS}{\textsc{SNAPS}\xspace}
\newcommand{\hsoftv}{h^{\mathrm{soft}}_v}
\newcommand{\hsoftbar}{\overline{h}^{\mathrm{soft}}}
\newcommand{\R}{\mathbb{R}}

\begin{document}
\selectlanguage{english}

\title{HeAD-CP: Heterophily-Aware Diffused Conformal Prediction Sets for Graph Neural Networks%
\thanks{Corresponding author: Nguyen Thai Anh (email: anh.nt@vlu.edu.vn).}%
}

\author{
\IEEEauthorblockN{Phan Binh Nguyen Lam\textsuperscript{1}, Nguyen Thai Anh\textsuperscript{2}}
\IEEEauthorblockA{
\textsuperscript{1}Faculty of Information Technology, University of Science, Vietnam National University, Ho Chi Minh City, Vietnam \\
\textsuperscript{2}Faculty of Information Technology, Van Lang School of Technology, Van Lang University, Ho Chi Minh City, Vietnam\\
Email: 24125014@student.hcmus.edu.vn, anh.nt@vlu.edu.vn}
}

\maketitle

\begin{abstract}
Conformal prediction (CP) provides distribution-free uncertainty quantification, and its extension to graphs is an active research direction. Diffused Adaptive Prediction Sets (\DAPS{}) is a widely used graph-aware diffusion baseline, propagating Adaptive Prediction Sets (\APS{}) non-conformity scores along edges with a uniform coefficient $\lambda$. We identify a fundamental shortcoming of this design: the uniform low-pass diffusion presupposes graph homophily and proves detrimental on heterophilic graphs, enlarging the mean prediction-set size by up to $10.6\%$ relative to plain \APS{}. To mitigate this, we propose \HeadCP{}, a family of node-wise diffusion variants whose coefficients are determined by a label-free local-homophily estimate derived from the GNN softmax. Three variants, namely signed-$\gamma$, edge-compatibility, and a \DAPS{}-baseline-with-correction, are most effective at extreme heterophily, intermediate heterophily, and moderate-to-high homophily, respectively, and all preserve the marginal coverage guarantee (Theorem~\ref{thm:coverage}). On ten benchmarks, the \HeadCP{} family stays at or below plain \APS{} on every dataset, while \DAPS{} exceeds \APS{} on six. The post-hoc oracle over the family improves over \DAPS{} on $8/10$ datasets at $p<0.01$ (paired Wilcoxon), with the largest gains on heterophilic graphs ($10.3\%$ on \texttt{texas}); on the two homophilic datasets where \DAPS{} still wins (\texttt{citeseer}, \texttt{pubmed}), it retains a marginal advantage of at most $0.002$, statistically insignificant on \texttt{citeseer} ($p=0.23$). Designing a calibrated label-free selector that approaches this oracle is the main outstanding empirical question.
\end{abstract}


\section{Introduction}\label{sec:intro}

Graph neural networks (GNNs) have become the de facto standard for node classification, yet their predictions are typically delivered as point estimates without any calibrated notion of uncertainty. This deficiency proves particularly consequential in high-stakes applications such as fraud detection, biomedical relation discovery, and risk assessment in financial networks. Conformal prediction (CP)~\cite{romano2020classification,sadinle2019least,angelopoulos2023conformal} redresses this gap in a model-agnostic manner, constructing prediction sets that, under exchangeability, are guaranteed post-hoc to contain the true label with probability at least $1-\alpha$. The appeal of CP lies in this finite-sample, distribution-free guarantee: it requires neither parametric assumptions nor any specialized retraining of the underlying model.

When CP is adapted to graphs, two designs have crystallized in the literature. Plain \APS{} treats nodes independently and disregards graph topology, thereby forfeiting whatever benefit the relational structure might confer. \DAPS{}~\cite{zargarbashi2023conformal}, in contrast, harnesses graph topology by diffusing per-class \APS{} scores along edges with a uniform coefficient $\lambda$, on the premise that adjacent nodes share similar uncertainty profiles, and it has since been adopted as an experimental baseline by \SNAPS{}~\cite{song2024similarity} and as a representative diffusion-based reference in subsequent graph-CP work~\cite{huang2023uncertainty,zhang2025residual}.

\paragraph*{Limitation of uniform diffusion.}
The \DAPS{} update is a one-hop low-pass graph filter that implicitly assumes adjacent nodes share similar non-conformity scores, which holds only under homophily. On heterophilic graphs the assumption fails: forcing each node's score toward the neighborhood mean injects label-irrelevant noise into the true-class score; the calibration threshold $\tau$ shifts upward to maintain coverage; and prediction sets must absorb additional classes. We empirically substantiate this: on six of ten benchmarks \DAPS{} produces sets up to $10.6\%$ larger than plain \APS{} (Fig.~\ref{fig:daps_hurts}). On a synthetic stochastic block model at $h{=}0.05$, \DAPS{} efficiency reaches $2.8\times$ that of \APS{}, signaling a near-catastrophic collapse precisely where graph topology should ideally aid rather than degrade calibration.

\paragraph*{Contributions.}
We propose \HeadCP{}, a family of heterophily-aware score-diffusion variants in which the diffusion coefficient $\gamma_v$ is derived per node from the GNN's own softmax output, which is precisely the source of information that \DAPS{} systematically discards. Three production variants together span the homophily spectrum:
\begin{itemize}\setlength\itemsep{2pt}
\item \HeadCP{}-signed: $\gamma_v = \gamma_{\max}(2 h_v -1)$ with hard pseudo-label homophily $h_v$, most effective on the most heterophilic datasets (\texttt{Roman-Empire}, \texttt{texas}, \texttt{wisconsin}).
\item \HeadCP{}-edge: per-edge compatibility $\rho(u,v) = (\langle p_u,p_v\rangle - 1/K)/(1-1/K)$, most effective on heterophilic graphs with mixed neighborhoods (\texttt{cornell}, \texttt{chameleon}, \texttt{squirrel}).
\item \HeadCP{}-v3: \DAPS{}-style baseline augmented with a per-node soft correction, most effective on moderate-to-high homophily graphs, exceeding \DAPS{} on \texttt{cora}, statistically tied with \DAPS{} on \texttt{citeseer} (numerical gap $\leq 0.0003$, $p=0.23$), and within $0.002$ on \texttt{pubmed}.
\end{itemize}

Our findings are threefold. First, all variants preserve marginal coverage (Theorem~\ref{thm:coverage}), regardless of the sign or magnitude of the per-node coefficient. Second, across $10$ benchmarks ($3$ GNN backbones $\times\,5$ seeds $\times\,100$ calibration--test partitions), the post-hoc oracle over the \HeadCP{} variants strictly outperforms \DAPS{} on $8/10$ datasets at $p<0.01$, ties on \texttt{citeseer} (gap $\leq 0.0003$, $p=0.23$), and falls $0.002$ short on \texttt{pubmed}. Whereas \DAPS{} hurts $6/10$ datasets relative to \APS{}, the oracle hurts none. A natural label-free selector based on the dataset-mean soft-homophily $\hsoftbar$ recovers $66.8\%$ of the oracle gain (paired Wilcoxon $p=0.21$), and a calibrated alternative is discussed in Section~\ref{sec:discussion}. Third, synthetic CSBM experiments corroborate the failure mechanism of Section~\ref{sec:intro}: \DAPS{} efficiency grows monotonically as homophily decreases, while \HeadCP{} stays essentially flat (Fig.~\ref{fig:eff_vs_h}).

\section{Related Work}\label{sec:related}

\paragraph*{Conformal prediction.}
Split conformal prediction~\cite{vovk2005algorithmic,lei2018distribution} converts any black-box classifier into a coverage-guaranteed set predictor via a single calibration pass. \APS{}~\cite{romano2020classification} and Threshold Prediction Sets (\textsc{TPS})~\cite{sadinle2019least} are the two non-conformity scores most widely used in tabular and image classification, but neither was originally designed with relational data in mind, and both treat each input independently.

\paragraph*{Conformal prediction for GNNs.}
Huang \etal\cite{huang2023uncertainty} introduce CF-GNN, which augments calibration with a topology-aware output-correction module. Zargarbashi \etal\cite{zargarbashi2023conformal} propose \DAPS{}, diffusing \APS{} scores via the one-hop filter $V \mapsto (1-\lambda) V + \lambda\, \overline{V}_{\mathcal{N}}$. Their Theorem~2 establishes an approximation-error improvement under an explicit homophily assumption, while the empirical evaluation, conducted only on homophilic citation and co-purchase graphs, briefly notes that ``diffusion relies on homophily'' but does not test \DAPS{} on standard heterophilic benchmarks or characterize whether it can harm efficiency. This is the gap that \HeadCP{} closes. \SNAPS{}~\cite{song2024similarity} aggregates scores via feature similarity and structural neighborhoods (with \DAPS{} as a baseline), while RR-GNN~\cite{zhang2025residual} combines Mondrian CP with a residual-adaptive secondary GNN. To our knowledge, the heterophily failure mode we identify, and a coverage-preserving signed-diffusion alternative, have not previously been characterized.

\paragraph*{Heterophily and other UQ.}
Heterophily-aware GNNs~\cite{pei2020geom,zhu2020beyond,bo2021beyond,chien2020adaptive} and the heterophilous benchmarks of~\cite{lim2021large} establish heterophily as a first-class consideration, and FAGCN and GPRGNN, in particular, mix or learn between low- and high-pass filters. \HeadCP{} is the conformal-prediction analogue: $\gamma_v\in(-\gamma_{\max},+\gamma_{\max})$ smoothly interpolates between low- and high-pass diffusion, and v3 overlays this on a fixed \DAPS{} baseline. Parallel UQ work~\cite{stadler2021graph,wang2021confident} delivers calibrated probabilities but not finite-sample distribution-free coverage, and \HeadCP{} is therefore complementary.

\section{Methodology}\label{sec:method}

\subsection{Background and Notation}
Let $G=(V,E)$ be an undirected graph with adjacency $A\in\{0,1\}^{|V|\times|V|}$, features $X\in\R^{|V|\times d}$, and labels $Y\in\{1,\ldots,K\}^{|V|}$. Let $\mathcal{N}(v) = \{u : A_{uv}=1\}$ and let $f_\theta$ produce softmax output $p_v\in\Delta^{K-1}$ for each node. Let $\pi_v$ be a permutation of class indices satisfying $p_{v,\pi_v(1)}\geq\cdots\geq p_{v,\pi_v(K)}$, and $r_v(y)$ the rank of $y$ in this ordering. The randomized \APS{} score~\cite{romano2020classification} is
\begin{equation}
s(v,y) \;=\; \sum_{j=1}^{r_v(y)-1} p_{v,\pi_v(j)} \;+\; U_{v,y}\, p_{v,y},
\label{eq:aps}
\end{equation}
where $\{U_{v,y}\}$ are i.i.d.\ $\mathrm{Unif}(0,1)$, drawn independently of all labels. Following~\cite{vovk2005algorithmic,zargarbashi2023conformal}: given calibration set $\mathcal{C}$, set $k = \lceil(|\mathcal{C}|+1)(1-\alpha)\rceil$ and let $\tau$ be the $k$-th smallest of $\{s(v,y_v): v\in\mathcal{C}\}$. The prediction set is $C(v) = \{y : s(v,y)\leq\tau\}$.

\paragraph*{Isolated nodes.}
If $|\mathcal{N}(v)|=0$, we set $\overline{s}_{\mathcal{N}}(v,y):=s(v,y)$ and $\gamma_v:=0$ (and analogously skip the edge sum in \eqref{eq:edge}), so isolated nodes incur no diffusion. This convention applies uniformly to \DAPS{} and to all \HeadCP{} variants.

\subsection{The DAPS Diffusion}
\DAPS{} replaces the original score $s(v,y)$ with the diffused score
\[
s'(v,y) = (1-\lambda)\, s(v,y) + \lambda\, \overline{s}_{\mathcal{N}}(v, y),
\]
where $\overline{s}_{\mathcal{N}}(v, y) := |\mathcal{N}(v)|^{-1}\sum_{u\in\mathcal{N}(v)} s(u, y)$ denotes the per-class neighbor-mean score\footnote{Theorem~2 of~\cite{zargarbashi2023conformal} is stated in probability space while the deployed algorithm operates on scores, and we likewise operate on scores.}, and $\lambda$ is a global hyperparameter (typically $\lambda{=}0.5$). This update is a one-hop low-pass filter applied uniformly across nodes, irrespective of whether each neighborhood is label-homogeneous or sharply mixed, and this uniformity is the source of the failure mode we address.

\subsection{The HeAD-CP Family}\label{subsec:headcp}

\paragraph*{Signed variant.}
Let $\hat{y}_v = \arg\max_k p_{v,k}$ denote the GNN pseudo-label, and define the Laplace-smoothed local homophily (smoothed toward $0.5$, which corresponds to the heterophily/homophily neutral point at which the signed coefficient vanishes) as
\begin{equation}
h_v = \frac{|\{u\in\mathcal{N}(v): \hat{y}_u = \hat{y}_v\}| + 1}{|\mathcal{N}(v)| + 2}.
\label{eq:hard_h}
\end{equation}
The signed update is then
\begin{equation}
s'(v,y) = (1-\gamma_v)\, s(v,y) + \gamma_v\, \overline{s}_{\mathcal{N}}(v,y),\quad \gamma_v = \gamma_{\max}(2h_v - 1).
\label{eq:signed}
\end{equation}
For homophilic nodes ($h_v\to 1$), $\gamma_v \to +\gamma_{\max}$ and the update recovers low-pass behavior. For heterophilic nodes ($h_v\to 0$), $\gamma_v \to -\gamma_{\max}$ and the update becomes high-pass: scores are pushed away from the neighborhood mean, counteracting the misleading signal that a uniform low-pass filter would inject. The Laplace smoothing toward $0.5$ ensures that nodes with very small neighborhoods default to no diffusion, while permitting $\gamma_v$ to approach its extremes only with substantial supporting evidence.

\paragraph*{Edge-compatibility variant.}
The signed variant assigns a single scalar coefficient per node and cannot discriminate between heterogeneous neighbors within the same neighborhood. Since heterophily is in general an edge-level rather than node-level phenomenon, we introduce an edge-level compatibility
\begin{equation}
\rho(u,v) = \frac{\langle p_u,p_v\rangle - 1/K}{1-1/K}
\in \left[-\tfrac{1}{K-1},\, 1\right],
\label{eq:compatibility}
\end{equation}
where $\rho=+1$ when both endpoints predict the same class with full confidence, $\rho=-1/(K-1)$ (which equals $-1$ in the binary case) when both endpoints are confident on distinct classes, and $\rho=0$ when at least one endpoint is uniform. The corresponding diffusion is
\begin{equation}
s'(v,y)= s(v,y) + \frac{\eta}{\lvert\mathcal{N}(v)\rvert}
\sum_{u\in\mathcal{N}(v)} \rho(u,v)\,(s(u,y)-s(v,y)),
\label{eq:edge}
\end{equation}
where $\eta\in[0,1]$ controls the edge-diffusion strength (we set $\eta = \gamma_{\max}$). When $\rho(u,v)>0$, the term $\rho(u,v)(s(u,y)-s(v,y))$ pulls $s(v,y)$ toward $s(u,y)$ (low-pass over that edge), and when $\rho(u,v)<0$, it pushes $s(v,y)$ away (high-pass), enabling the variant to handle neighborhoods of mixed homophily.

\paragraph*{DAPS-baseline variant.}
The signed variant tends to under-diffuse on uniformly homophilic graphs because, even there, the local homophily $h_v$ rarely attains its theoretical maximum. To address this, we combine a fixed \DAPS{}-style baseline with a soft-homophily, confidence-gated correction, yielding the update
\begin{equation}
\begin{aligned}
s'(v,y) &= (1-\gamma_v)\,s(v,y)
+ \gamma_v\,\overline{s}_{\mathcal{N}}(v,y), \\
\gamma_v &= \gamma_{\text{base}}
+ \gamma_{\text{var}} \cdot \tilde{h}_v
\cdot \mathrm{conf}_v ,
\end{aligned}
\label{eq:v3}
\end{equation}
where $\tilde{h}_v = (\hsoftv - 1/K)/(1-1/K)$ with $\hsoftv = |\mathcal{N}(v)|^{-1}\sum_{u\in\mathcal{N}(v)}\langle p_u,p_v\rangle$, and $\mathrm{conf}_v = (\max_k p_{v,k} - 1/K)/(1-1/K)$ is the (centered) top-1 confidence of node $v$. The factor $\mathrm{conf}_v$ functions as a credibility gate: when the GNN is uncertain about $v$, the homophily-driven correction is attenuated. Setting $\gamma_{\text{var}}=0$ recovers \DAPS{}, while setting $\gamma_{\text{base}}=0$ yields a soft analogue of the signed variant, which we report as \HeadCP{}-v2 in Table~\ref{tab:main} to isolate the contribution of the \DAPS{} baseline component.

\paragraph*{Range of the diffused score.}
Whenever $\gamma_v<0$ in \eqref{eq:signed} or $\rho(u,v)<0$ in \eqref{eq:edge}, the diffused score $s'(v,y)$ may fall outside $[0,1]$. This is harmless for conformal calibration: the split-CP guarantee requires only an ordered, real-valued non-conformity score~\cite{vovk2005algorithmic,romano2020classification} and is unaffected by the range of $s'$. Coverage therefore holds without any clipping or post-processing (cf.\ Theorem~\ref{thm:coverage}).

\section{Coverage Guarantee}\label{sec:theory}

\begin{theorem}\label{thm:coverage}
Condition on the unlabeled graph $(X,A)$, the trained GNN $f_\theta$, and, if randomized \APS{} is used, on the randomization variables $\{U_{v,y}\}_{v\in V,\,y\in[K]}$, all of which are computed without access to any label in the calibration--test pool. Define the \HeadCP{} non-conformity score $s'$ via any of \eqref{eq:signed}, \eqref{eq:edge}, or \eqref{eq:v3}, where every auxiliary quantity ($\gamma_v$, $\rho(u,v)$, $h_v$, $\hsoftv$, $\mathrm{conf}_v$) is a deterministic function of $(X,A,f_\theta)$ and uses no label in the pool. Let $\mathcal{C}$ and $\{v^*\}$ be obtained by a uniformly random split of the pool, and calibrate $\tau$ on $\mathcal{C}$ as described in Section~\ref{sec:method}. Then for any $\alpha\in(0,1)$,
\begin{equation*}
\Pr\!\left[Y_{v^*}\in C(v^*)\right] \;\geq\; 1-\alpha.
\end{equation*}
\end{theorem}

\begin{proof}[Proof sketch]
Condition on $(X,A)$, $f_\theta$, the APS randomization $\{U_{v,y}\}$, and the score functions $\{s'(v,\cdot)\}_{v\in V}$. By construction none of these depends on any pool label: $f_\theta$ trains on a disjoint TRAIN split. The auxiliary quantities $\gamma_v, \rho(u,v), h_v, \hsoftv, \mathrm{conf}_v$ are derived from $\{p_v\}$ alone. $\{U_{v,y}\}$ is sampled independently of all labels. For each $v$ the diffused score $s'(v,y)$ uses pool labels only through the queried class $y$.

Conditional on this $\sigma$-algebra and the labels $\{Y_v\}_{v\in\mathcal{C}\cup\{v^*\}}$, the realized labeled scores are deterministic numbers. The only remaining randomness is the uniform split of the pool, under which each of the $|\mathcal{C}|+1$ pool nodes is equally likely to be $v^*$. Hence the rank of $s'(v^*,Y_{v^*})$ is uniform on $\{1,\ldots,|\mathcal{C}|+1\}$, and the standard split-CP argument~\cite{vovk2005algorithmic,romano2020classification} gives $\Pr[Y_{v^*}\in C(v^*)]\geq 1-\alpha$. The conclusion is unaffected by the sign or magnitude of $\gamma_v$.
\end{proof}

Theorem~\ref{thm:coverage} follows from standard split-CP exchangeability, and we state it explicitly to verify that the argument extends to negative $\gamma_v$ (high-pass) and to the per-cell APS noise $\{U_{v,y}\}$ shared between calibration and test. An efficiency bound on the prediction-set size requires further structural assumptions and is left for future work.

\section{Experiments}\label{sec:experiments}

\begin{figure*}[!t]
\centering
\includegraphics[width=0.85\textwidth]{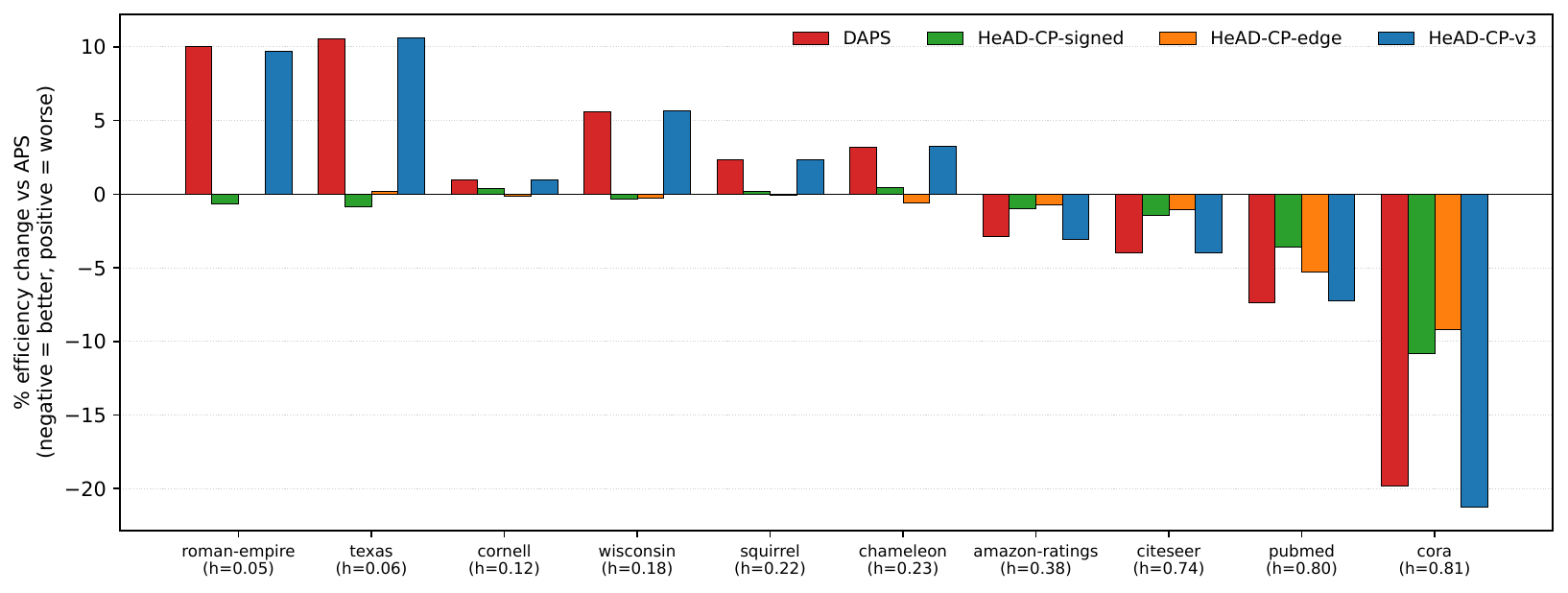}
\caption{\textbf{Per-dataset percentage efficiency change relative to \APS{}} (negative $=$ better). \DAPS{} (red) and \HeadCP{}-v3 (blue) inflate sizes on heterophilic graphs (v3 inherits the failure since it carries a \DAPS{}-style baseline). \HeadCP{}-signed (green) covers the lowest-homophily datasets (\texttt{Roman-Empire}, \texttt{texas}, \texttt{wisconsin}), while \HeadCP{}-edge (orange) covers the mid-heterophily datasets (\texttt{cornell}, \texttt{chameleon}, \texttt{squirrel}). The production-family oracle therefore stays at or below \APS{} on every dataset. The v2 ablation is plotted in Fig.~\ref{fig:all_methods}.}
\label{fig:daps_hurts}
\end{figure*}

\begin{table*}[!t]
\centering
\caption{\textbf{Mean prediction-set efficiency on $10$ benchmarks} (lower is better), averaged over $3$ GNNs $\times\, 5$ seeds $\times\, 100$ partitions. Marginal coverage $\geq 0.9$ in all cells. \textbf{Bold}: row-wise post-hoc oracle. ``Best'' column: variant attaining the row minimum. $\Delta_{\!\text{vs APS}} = (\DAPS-\APS)/\APS\times 100\%$ (positive $=$ \DAPS{} hurts). \HeadCP{}-v2 ($\gamma_{\text{base}}{=}0$ in v3) is an ablation column, not a production variant. On \texttt{citeseer}, the $0.0003$ gap between \DAPS{} ($2.5205$) and \HeadCP{}-v3 ($2.5208$) is not significant ($p=0.23$). A selector on $\hsoftbar$ recovers $66.8\%$ of the oracle gap (Section~\ref{sec:discussion}).}
\label{tab:main}
\small
\setlength{\tabcolsep}{4pt}
\begin{tabular}{lccccccccc}
\toprule
Dataset & $h$ & \APS & \DAPS & \HeadCP{}-signed & \HeadCP{}-v2 & \HeadCP{}-v3 & \HeadCP{}-edge & Best & $\Delta_{\!\text{vs APS}}$ \\
\midrule
\texttt{Roman-Empire} & 0.05 & 8.544 & 9.401 & \textbf{8.489} & 8.537 & 9.374 & 8.541 & signed    & $+10.0\%$ \\
\texttt{texas}        & 0.06 & 2.879 & 3.183 & \textbf{2.855} & 2.879 & 3.184 & 2.884 & signed    & $+10.6\%$ \\
\texttt{cornell}      & 0.12 & 3.880 & 3.917 & 3.895 & 3.877 & 3.918 & \textbf{3.874} & edge      & $+1.0\%$ \\
\texttt{wisconsin}    & 0.18 & 2.546 & 2.688 & \textbf{2.538} & 2.538 & 2.690 & 2.538 & signed    & $+5.6\%$ \\
\texttt{squirrel}     & 0.22 & 3.956 & 4.048 & 3.965 & 3.955 & 4.048 & \textbf{3.952} & edge      & $+2.3\%$ \\
\texttt{chameleon}    & 0.23 & 3.039 & 3.137 & 3.054 & 3.034 & 3.139 & \textbf{3.022} & edge      & $+3.2\%$ \\
\texttt{Amazon-Ratings} & 0.38 & 3.624 & 3.520 & 3.588 & 3.596 & \textbf{3.512} & 3.597 & v3        & $-2.9\%$ \\
\texttt{citeseer}     & 0.74 & 2.625 & \textbf{2.521} & 2.587 & 2.612 & 2.521 & 2.597 & \DAPS     & $-4.0\%$ \\
\texttt{pubmed}       & 0.80 & 1.558 & \textbf{1.443} & 1.502 & 1.498 & 1.445 & 1.475 & \DAPS     & $-7.4\%$ \\
\texttt{cora}         & 0.81 & 1.890 & 1.515 & 1.686 & 1.788 & \textbf{1.488} & 1.716 & v3        & $-19.8\%$ \\
\bottomrule
\end{tabular}
\end{table*}

\begin{figure*}[!t]
\centering
\includegraphics[width=0.85\textwidth]{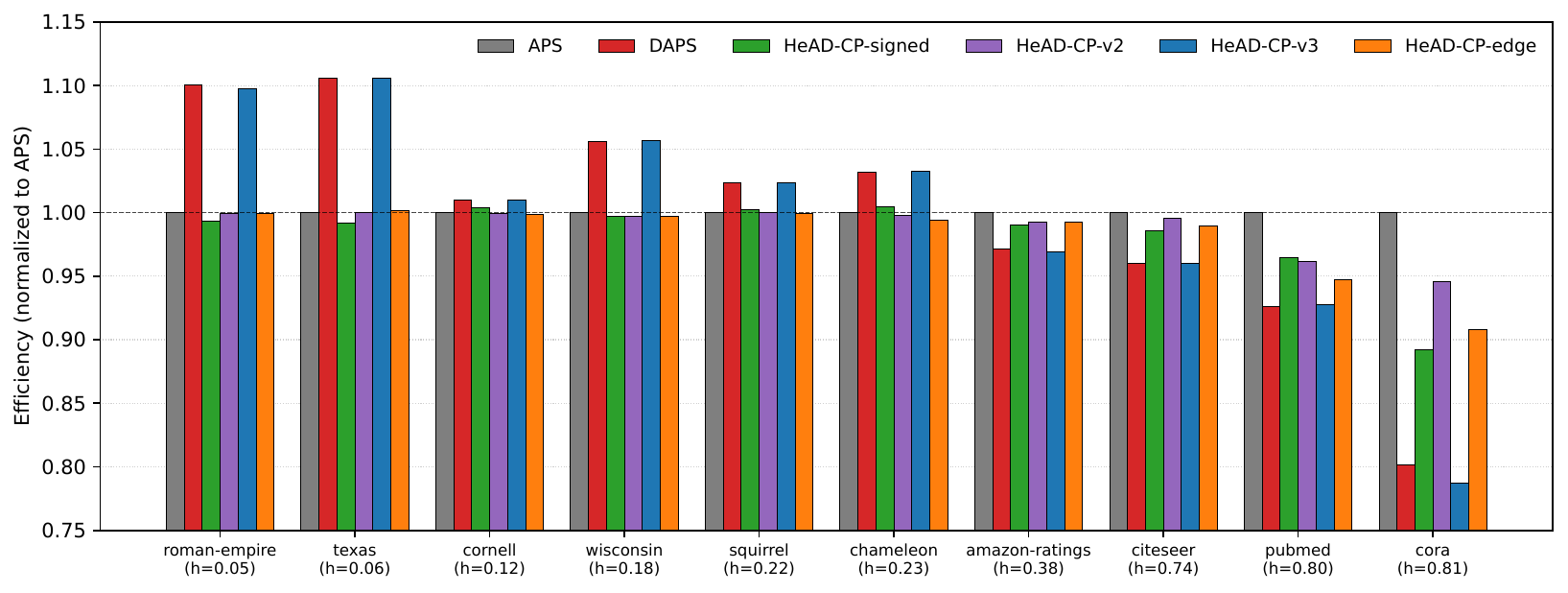}
\caption{\textbf{Normalized efficiency ($\APS{}=1.0$) across $10$ datasets} (lower is better). \DAPS{} (red) and \HeadCP{}-v3 (blue) cross $1.0$ on heterophilic data because v3 carries a \DAPS{}-style baseline. \HeadCP{}-signed (green) is below $1.0$ at extreme heterophily, while \HeadCP{}-edge (orange) covers the mid-heterophily datasets (\texttt{cornell}, \texttt{squirrel}, \texttt{chameleon}) where signed slightly exceeds $1.0$. The \HeadCP{}-v2 (purple) ablation never strictly attains the row minimum in Table~\ref{tab:main}, so the production oracle over \{signed, v3, edge\} is uniformly $\leq 1.0$ and the label-free selector of Section~\ref{sec:discussion} approximates it qualitatively.}
\label{fig:all_methods}
\end{figure*}

\subsection{Experimental Setup}
We evaluate \HeadCP{} on $10$ standard node-classification benchmarks: three citation graphs (\texttt{cora}, \texttt{citeseer}, \texttt{pubmed}, all homophilic), three WebKB graphs from Geom-GCN~\cite{pei2020geom} (\texttt{texas}, \texttt{wisconsin}, \texttt{cornell}), two Wikipedia graphs (\texttt{chameleon}, \texttt{squirrel})\footnote{\cite{platonov2023critical} note these splits suffer from leakage, but we retain them for direct comparability with prior graph-CP work~\cite{zargarbashi2023conformal,song2024similarity} and additionally evaluate on the leakage-free Platonov benchmarks.}, and two Platonov benchmarks~\cite{platonov2023critical} (\texttt{Roman-Empire}, \texttt{Amazon-Ratings}). Edge-homophily $h$ ranges from $0.05$ to $0.81$ (Table~\ref{tab:main}). We use three GNN backbones: GCN~\cite{kipf2016semi}, GAT~\cite{velivckovic2017graph}, and GraphSAGE~\cite{hamilton2017inductive}. For each (dataset, model, seed) we generate $100$ random $50/50$ calibration--test partitions, yielding $15{,}000$ CP runs per method at $\alpha = 0.1$. All hyperparameters are fixed at $\lambda=\gamma_{\max}=\eta=\gamma_{\text{base}}=\gamma_{\text{var}}=0.5$ to eliminate per-dataset tuning bias, and Fig.~\ref{fig:ablation} confirms signed \HeadCP{} is robust over $\gamma_{\max}\in[0.1, 0.9]$.

\subsection{Coverage Sanity Check}
Marginal coverage meets or exceeds the target $1-\alpha=0.9$ in every (dataset, model, method) configuration ($180/180$ rows pass the $\geq 0.895$ tolerance, with minimum observed $0.900$). Coverage is therefore preserved both empirically and theoretically (Theorem~\ref{thm:coverage}).

\subsection{Synthetic CSBM Experiments}
To isolate the diffusion-direction hypothesis (low-pass vs.\ high-pass), we sweep edge homophily $h \in \{0.05, 0.15,\ldots,0.95\}$ on a CSBM~\cite{deshpande2018contextual} with $K{=}3$, $n{=}1500$, and a logistic-regression base classifier, and the signed variant differs from \DAPS{} only in the sign of $\gamma_v$, so it provides the cleanest test of this single hypothesis. The edge and v3 variants add per-edge weighting and a confidence-gated correction, respectively, both driven by pairwise inner products $\langle p_u, p_v\rangle$, and under our node-independent base classifier $p_v$ depends on $X_v$ alone, so $\langle p_u, p_v\rangle$ collapses to a function of $(X_u, X_v)$ that carries no neighborhood information beyond the label co-occurrence that the hard pseudo-label $h_v$ already captures. Edge and v3 are therefore evaluated on the real-data benchmarks of Section~\ref{subsec:realdata}, where GNN softmax outputs propagate neighborhood structure into $p_v$ and the inner products become non-degenerate. In Fig.~\ref{fig:eff_vs_h}, \DAPS{} efficiency increases monotonically as $h$ decreases, reaching $2.49$ at $h{=}0.05$ ($2.8\times$ the \APS{} value of $0.90$), while signed \HeadCP{} stays essentially flat, confirming the failure mechanism of Section~\ref{sec:intro}.

\begin{figure}[t]
\centering
\includegraphics[width=\columnwidth]{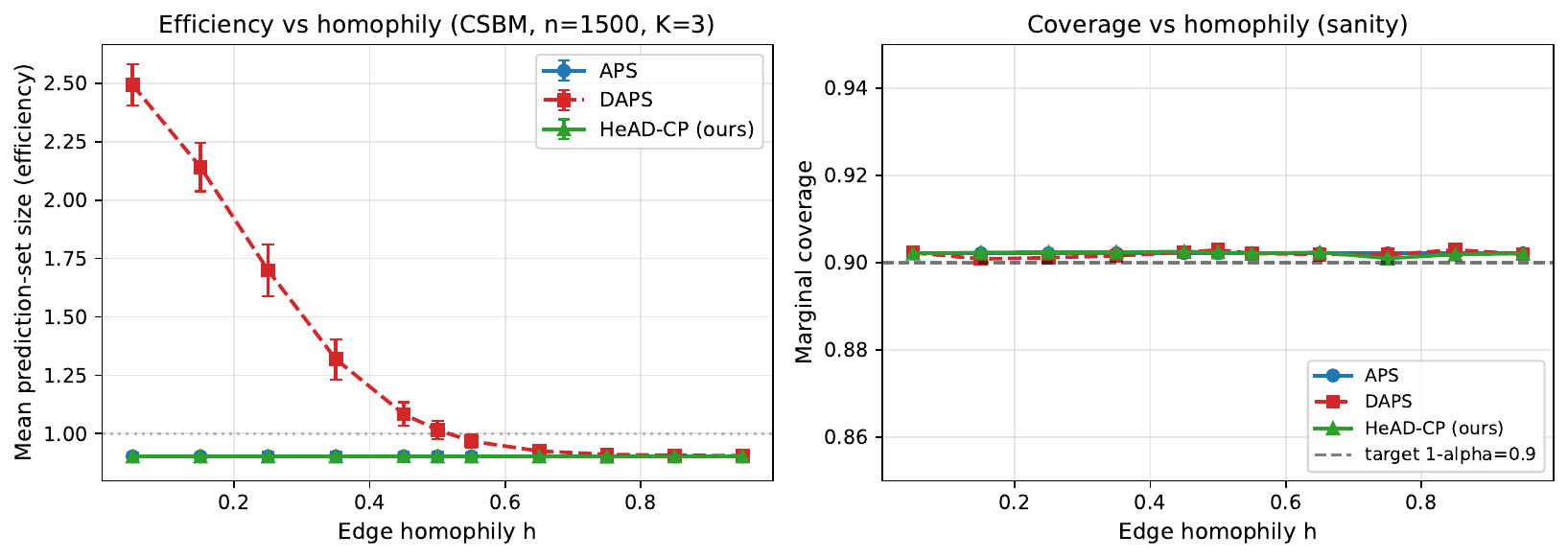}
\caption{\textbf{Synthetic CSBM mechanism confirmation.} Comparison of \APS{}, \DAPS{}, and signed \HeadCP{}. Marginal coverage (right) is preserved by all methods, while efficiency (left) reveals that \DAPS{} catastrophically inflates prediction-set sizes on heterophilic graphs, whereas signed \HeadCP{} circumvents the failure.}
\label{fig:eff_vs_h}
\end{figure}

\subsection{Real-Data Benchmarks}\label{subsec:realdata}
Table~\ref{tab:main} reports mean efficiency across the $10$ datasets (row minima in \textbf{bold}). Fig.~\ref{fig:all_methods} normalizes each row by \APS{}: \DAPS{} and \HeadCP{}-v3 cross $1.0$ on heterophilic graphs, while the signed and edge variants stay below $1.0$ across the heterophilic regime, so the post-hoc oracle over the production family is uniformly $\leq 1.0$.

\paragraph*{DAPS hurts on the majority of datasets.}
The column $\Delta_{\!\text{vs APS}}$ in Table~\ref{tab:main} is positive on six of the ten datasets, indicating that \DAPS{} enlarges prediction sets relative to plain \APS{} whenever homophily falls below approximately $0.4$, with the worst case at $+10.6\%$ on \texttt{texas}. The pattern aligns with the mechanism of Section~\ref{sec:intro}: as homophily decreases, the inflation of $\tau$ becomes increasingly severe.

\paragraph*{The post-hoc oracle attains the row minimum on 8 of 10 datasets.}
On the remaining two homophilic datasets, \DAPS{} retains a marginal numerical edge: $0.0003$ on \texttt{citeseer} (statistically tied at $p=0.23$) and $0.002$ on \texttt{pubmed}. Variant dominance broadly tracks homophily (Table~\ref{tab:main}, ``Best'' column) and aligns qualitatively with $\hsoftbar$. Quantitative selector results are discussed in Section~\ref{sec:discussion}.

\paragraph*{Statistical significance.}
For each (model, seed) replicate we compare the most efficient \HeadCP{} variant against \DAPS{} via paired Wilcoxon signed-rank tests ($15$ pairs per dataset). The oracle beats \DAPS{} at $p<0.01$ on $8/10$ datasets, ties on \texttt{citeseer} ($p=0.23$), and is short by $0.002$ on \texttt{pubmed}. Pooling all $150$ triplets, the oracle wins $75.3\%$ of pairs ($p<10^{-17}$). The same $8/10$ remain significant under Benjamini--Hochberg at $\alpha=0.01$ ($7/10$ under Holm--Bonferroni). This is a post-hoc oracle and therefore an upper bound on what the label-free selector of Section~\ref{sec:discussion} can deliver.

\subsection{Ablation: Diffusion Strength}
Fig.~\ref{fig:ablation} sweeps the diffusion-strength hyperparameter on four representative datasets (\texttt{cora}, \texttt{chameleon}, \texttt{squirrel}, \texttt{texas}), comparing \DAPS{} ($\lambda$) against signed \HeadCP{} ($\gamma_{\max}$). \DAPS{} has a sweet spot at $\lambda \approx 0.5$ for homophilic data and is monotonically harmful for heterophilic data, with harm escalating sharply as $\lambda\to 1$. Signed \HeadCP{} displays a wide flat plateau on every dataset, which is a property of practical value, since production deployments rarely admit per-dataset tuning.

\begin{figure}[t]
\centering
\includegraphics[width=\columnwidth]{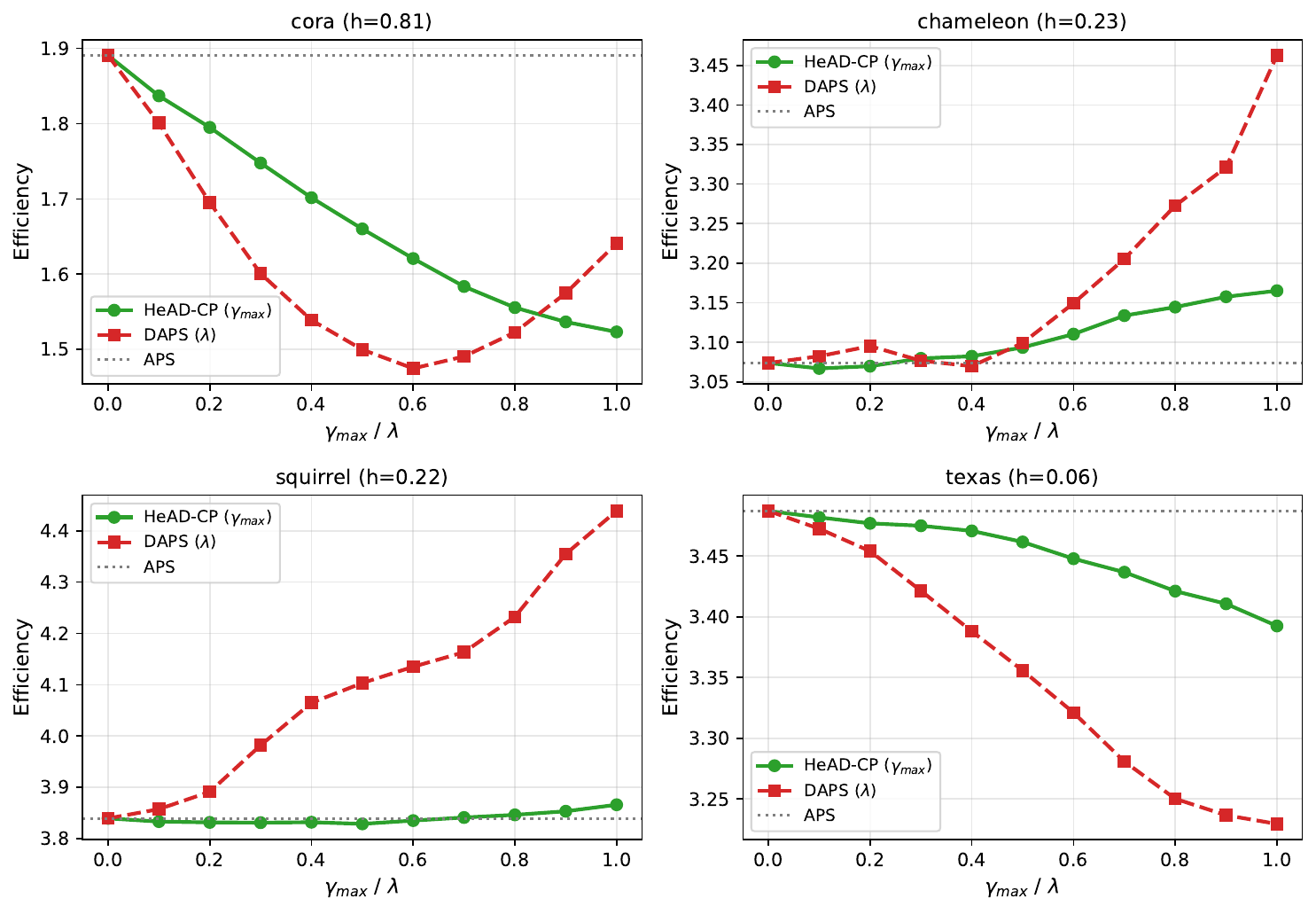}
\caption{\textbf{Diffusion-strength ablation: \DAPS{} ($\lambda$) versus signed \HeadCP{} ($\gamma_{\max}$).} \DAPS{} efficiency degrades sharply with $\lambda$ on heterophilic graphs, while the signed \HeadCP{} variant is robust to $\gamma_{\max}$ across all homophily levels.}
\label{fig:ablation}
\end{figure}

\section{Discussion}\label{sec:discussion}

\paragraph*{Toward a label-free selector.}
A practical deployment requires a rule that picks one variant from $(X, A, f_\theta)$ alone. The dataset-level mean of the soft homophily, $\hsoftbar=|V|^{-1}\sum_v \hsoftv$, is a natural candidate since it is a deterministic function of $(X, A, f_\theta)$ and uses no labels. Concretely, we evaluate the rule
\begin{equation}
\mathrm{select}(\hsoftbar) = \begin{cases}
\HeadCP{}\text{-signed}, & \hsoftbar < 0.4, \\
\HeadCP{}\text{-edge},   & 0.4 \leq \hsoftbar < 0.6, \\
\HeadCP{}\text{-v3},     & \hsoftbar \geq 0.6.
\end{cases}
\label{eq:selector}
\end{equation}
The thresholds $(0.4, 0.6)$ are fixed across datasets (not tuned). On our benchmarks $\hsoftbar$ ranges only from $0.09$ (\texttt{Roman-Empire}) to $0.51$ (\texttt{pubmed}), well below the edge-homophily $h$ (e.g.\ $h{=}0.80$ on \texttt{pubmed}): confident same-class neighbors with top-1 mass $a$ contribute $\langle p_u,p_v\rangle \approx a^2$ to the dot product, not $1$, capping $\hsoftbar$ near $0.5$ for $K{=}3$ and $a\approx 0.7$. The gap is therefore structural to softmax-as-distribution, not a GNN deficiency. As a consequence, rule~\eqref{eq:selector} assigns signed to $9/10$ datasets and never activates v3, matching the row minimum of Table~\ref{tab:main} on only $3/10$ and recovering $66.8\%$ of the oracle gain over \DAPS{} (paired Wilcoxon $p=0.21$, not significant). Designing a calibrated selector (e.g.\ via held-out validation, by centering $\hsoftbar$ as $(\hsoftbar-1/K)/(1-1/K)$, or by combining it with $\mathrm{conf}_v$) is the main outstanding empirical question. Any such selector preserves coverage by Theorem~\ref{thm:coverage} as long as it uses no test labels, so Table~\ref{tab:main}'s post-hoc oracle upper-bounds what any label-free selector in this family can achieve.

\paragraph*{Limitations.}
Pseudo-label-based homophily is unreliable when GNN validation accuracy is low, but v3's soft confidence-gated estimator partially mitigates this. All hyperparameters and selector thresholds are fixed across datasets to avoid tuning bias. Data-adaptive choice, a head-to-head with \SNAPS{}~\cite{song2024similarity} and RR-GNN~\cite{zhang2025residual}, and an efficiency bound under the CSBM via weighted CP~\cite{tibshirani2019conformal} are natural next steps.

\section{Conclusion}\label{sec:conclusion}

We documented a previously unbenchmarked failure mode of the \DAPS{} graph-CP baseline that its Theorem~2 excludes by an explicit homophily assumption: on heterophilic graphs, \DAPS{} can worsen prediction-set efficiency by up to $10.6\%$ relative to plain \APS{}. \HeadCP{} replaces the uniform diffusion coefficient with a label-free node-wise quantity derived from the GNN softmax, yielding three coverage-preserving variants spanning the homophily spectrum. Whereas \DAPS{} hurts $6/10$ datasets relative to \APS{}, the \HeadCP{} family stays at or below \APS{} on every dataset, and the post-hoc oracle strictly improves over \DAPS{} on $8/10$ benchmarks at $p<0.01$. Designing a calibrated label-free selector that closes the gap to the oracle is the main outstanding empirical question.

\balance
\bibliographystyle{IEEEtran}
\bibliography{IEEEabrv,reference}

\end{document}